\documentclass[11pt,twoside]{article}
\usepackage{fullpage}
\usepackage{epsf}
\usepackage{fancyhdr}
\usepackage{graphics}
\usepackage{graphicx}
\usepackage{enumitem}

\usepackage{psfrag}
\usepackage[T1]{fontenc}
\usepackage{color}
\usepackage{amsthm}
\usepackage{amsfonts}
\usepackage{amsmath}
\usepackage{amssymb}
\usepackage{mathrsfs}
\usepackage{bm}
\usepackage{accents}
\usepackage{listings}
\usepackage{caption}
\usepackage{subcaption}
\usepackage[linesnumbered, ruled, vlined]{algorithm2e}
\usepackage[titletoc,toc,title]{appendix}
\newlength{\widebarargwidth}
\newlength{\widebarargheight}
\newlength{\widebarargdepth}

\makeatletter
\long\def\@makecaption#1#2{
        \vskip 0.8ex
        \setbox\@tempboxa\hbox{\small {\bf #1:} #2}
        \parindent 1.5em  
        \dimen0=\hsize
        \advance\dimen0 by -3em
        \ifdim \wd\@tempboxa >\dimen0
                \hbox to \hsize{
                        \parindent 0em
                        \hfil 
                        \parbox{\dimen0}{\def\baselinestretch{0.96}\small
                                {\bf #1.} #2
                                } 
                        \hfil}
        \else \hbox to \hsize{\hfil \box\@tempboxa \hfil}
        \fi
        }
\makeatother
\usepackage{amssymb}
\usepackage{amsmath}
\usepackage{amsthm}

\newtheorem{theorem}{Theorem}
\newtheorem{lemma}[theorem]{Lemma}
\newtheorem{corollary}[theorem]{Corollary}
\newtheorem{definition}[theorem]{Definition}
\newtheorem{example}[theorem]{Example}

\usepackage[dvipsnames]{xcolor}
\usepackage{verbatim} 
\usepackage[titletoc,toc,title]{appendix}
\usepackage{csquotes} 
\usepackage{upquote} 
\usepackage[bottom]{footmisc} 
\usepackage{graphicx} 
\usepackage[ruled]{algorithm2e}

\usepackage[dvipsnames]{xcolor}
\usepackage{tikz}
\usetikzlibrary{matrix,shapes,arrows,positioning,chains, calc,backgrounds}
\usepackage[english]{babel} 
\renewcommand{\baselinestretch}{1.04} 
\date{}
\usepackage[style = alphabetic,citestyle=alphabetic,maxbibnames=99,backend=biber,sorting=nyt,natbib=true,backref=true]{biblatex}
\DefineBibliographyStrings{english}{%
  backrefpage = {Cited on page},
  backrefpages = {Cited on pages},
}
\usepackage[colorlinks,linkcolor = RawSienna, urlcolor  = RedViolet, citecolor = RoyalBlue, anchorcolor = ForestGreen,bookmarks=False]{hyperref}
\usepackage{cleveref}

\newcommand{\sign}{\mathrm{sign}}

\newcommand{\ty}{\tilde{y}}
\newcommand{\tx}{\tilde{x}}

\newcommand{\cC}{\mathcal{C}}

\newcommand{\cN}{\mathcal{N}}

\newcommand{\cS}{{\cal S}}

\newcommand{\R}{\mathbb{R}}

\newcommand{\argmin}{\operatornamewithlimits{argmin}}

\newcommand{\E}{\mathbb{E}}

\renewcommand{\Pr}{\mathbb{P}}
\newcommand{\lv}{\lVert}
\newcommand{\rv}{\rVert}
\newcommand{\vt}{v^{(t)}}
\newcommand{\vtone}{v^{(t+1)}}
\newcommand{\vzero}{v^{(0)}}
\newcommand{\At}{A_t}
\newcommand{\Atmax}{A_t^{\text{max}}}
\newcommand{\Atone}{A_{t+1}}

\newcommand{\Gt}{G_{t}}
\newcommand{\Ht}{H_{t}}
\newcommand{\Rt}{R(\vt)}
\newcommand{\Sample}{\mathcal{S}}
\newcommand{\sfQ}{\mathsf{Q}}
\newcommand{\sfP}{\mathsf{P}}
\newcommand{\tsfP}{\tilde{\mathsf{P}}}

\newcommand{\sfN}{\mathsf{N}}
\DeclareMathSizes{9}{8}{7}{5}

\title{\textbf{Finite-sample Analysis 
       of Interpolating Linear Classifiers 
       in the Overparameterized Regime }}
\author{Niladri S. Chatterji \\ University of California, Berkeley \\ \textsf{chatterji@berkeley.edu} \and Philip M. Long \\ Google \\ \textsf{plong@google.com}}
\date{\today}
\begin{document}
\maketitle

\begin{abstract}
We prove bounds on the population risk of the maximum margin algorithm
for two-class linear classification.  For linearly separable training
data, the maximum margin algorithm has been shown in previous work to
be equivalent to a limit of training with logistic loss using gradient
descent, as the training error is driven to zero.  We analyze this
algorithm applied to random data including misclassification noise.
Our assumptions on the clean data include the case in which the
class-conditional distributions are standard normal distributions.
The misclassification noise may be chosen by an adversary, subject to a
limit on the fraction of corrupted labels.  Our bounds show that, with
sufficient over-parameterization, 
the maximum margin algorithm trained on
noisy data can achieve nearly optimal population risk.
\end{abstract}

\section{Introduction}
A surprising statistical phenomenon 
has emerged in modern machine
learning:
highly complex models can interpolate 
%
training data while still generalizing well to test
data, even in the presence of label noise. 
This is rather striking as it the goes against the grain of the
classical statistical wisdom which dictates that predictors that
generalize well should 
trade off
between the fit to the training data
and the 
some measure of the complexity or smoothness
of the predictor. 
Many estimators like
neural networks,
kernel estimators, nearest neighbour
estimators, and even linear models have been shown to demonstrate this
phenomenon~\citep[see,][among others]{zhang2016understanding,belkin2019reconciling}.

This phenomenon
has recently
inspired
intense theoretical research.
One line of work 
\citep{soudry2018implicit,ji2018risk,gunasekar2017implicit,nacson2018stochastic,gunasekar2018implicit,gunasekar2018characterizing} 
formalized the argument \citep{neyshabur2014search,neyshabur2017implicit}
that, even when there is no explicit regularization that is used in training these rich models, 
there is nevertheless implicit regularization encoded in the choice of
the optimization method used.  For example, in the setting of linear
classification,
\citet{soudry2018implicit}, \citet{ji2018risk} and \citet{nacson2018stochastic} 
show that learning a linear classifier using gradient descent on the
unregularized logistic or exponential loss asymptotically leads the
solution to converge to the maximum $\ell_2$-margin classifier.
More concretely, given $n$ linearly separable samples $(x_i,y_i)_{j=1}^n$, where $x_i \in \mathbb{R}^p$ are the features and $y_i \in \{-1,1\}$, the iterates of gradient descent (initialized at the origin) are given by, 
\begin{align*}
\vtone := \vt -\alpha \nabla R_{\log}(\vt) \; \text{ where } \; R_{\log}(v):= \sum_{i=1}^n \log\left(1+\exp\left(-y_i (v\cdot x_i) \right)\right).
\end{align*}
They show that in the large-$t$ limit the normalized predictor obtained by gradient descent $\vt/\lv \vt\rv$ converges to $w/\lv w\rv$ where,
\begin{align} \label{def:maxmargin}
&w 
 =
 \argmin_{u \in \mathbb{R}^p} \; \lv u\rv,\\
\quad \text{such that, } \; &y_i (u\cdot x_i) \ge 1, \quad \text{for all } i\in [n]. \nonumber
\end{align}
That is, $w$ is the maximum $\ell_2$-margin classifier over the training data. 

The question still remains, though, why do these maximum margin classifiers
generalize well beyond the training set, 
despite the fact that they ``fit the noise''?
The fact that
$p > n$ renders traditional distribution-free
bounds~\citep{cover1965geometrical,vapnik1982estimation} 
vacuous.  Due to the presence of label noise,
margin bounds 
\citep{vapnik1995nature,shawe1998structural} are
also not an obvious answer.  

%


\begin{figure}[t]
\centering
\includegraphics[scale=0.8]{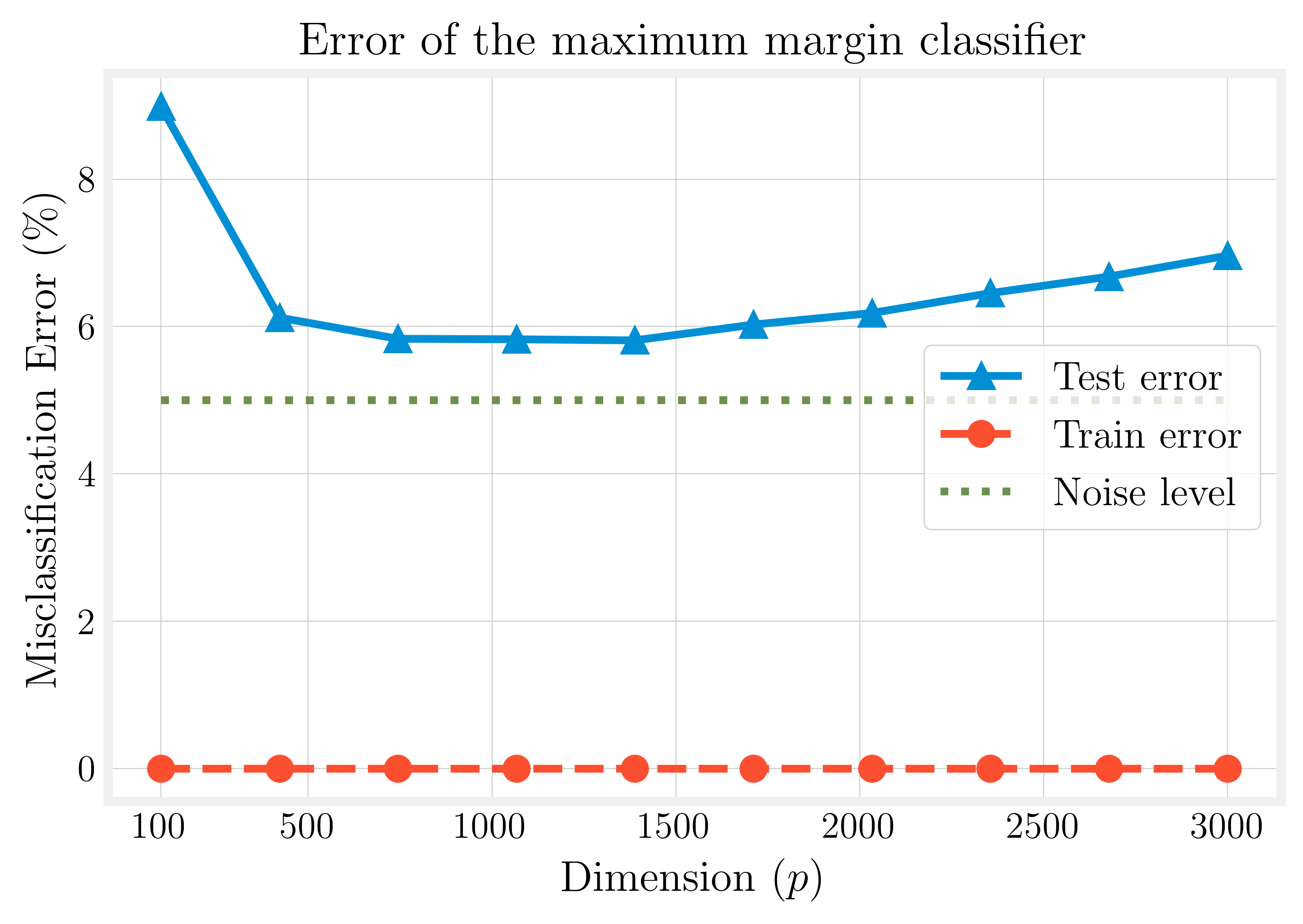}
\caption{\label{fig:introfigure}Plot of the test error (solid, blue) and train error (dashed, red) versus the dimension of the covariates $p$. The number of samples $n=100$ is kept fixed. The dimension $p$ is varied in the interval $[100,3000]$. The data is generated according to the Boolean noisy rare-weak model (see Section~\ref{s:defs}). First, a clean label $\ty$ is drawn by randomly flipping a fair coin. The covariates $x$ are drawn conditioned on $\ty$. The first $100$ attributes, $(x_1,\ldots,x_{100})$ are equal to the clean label $\ty$ with probability $0.7$, the remaining attributes $(x_{101},\ldots,x_{p})$ are either $-1$ or $1$ with equal probability. The noisy label sample $y$ is generated by flipping the true label $\ty$ with probability $\eta = 0.05$. The classifier is the maximum $\ell_2$-margin classifier defined in equation~\eqref{def:maxmargin}. The plot is generated by averaging over $500$ draws of the samples. The train error on all runs was always $0$.}
\end{figure}

In this paper, 
we prove an upper bound on the misclassification test error for the
maximum margin linear classifier, and therefore on
the 
the limit of gradient descent on the training error
without any complexity penalty.
Our analysis holds under a natural
and fairly general
generative model for the data.  
One special case is where adversarial label noise
\citep[see][]{kearns1994toward,kalai2008agnostically,
klivans2009learning,awasthi2017power,DBLP:conf/nips/Talwar20}
is added to data in which
the
positive examples are distributed as
$\mathsf{N}(\mu, I)$ and the negative examples
are distributed $\mathsf{N}(-\mu, I)$.  If $\lv \mu \rv$ is not too small,
the clean data will consist of overlapping but largely separate
clouds of points.  Our assumptions are 
weaker than this, however (see Section~\ref{s:defs} for the
details).
They
are satisfied by the case in which misclassification noise
is added to 
the generative model underlying Fisher's linear discriminant
\citep[see][]{duda2012pattern,hastie2009elements}
(except that, to make the analysis cleaner, the distribution
is shifted so that the origin is halfway between the class-conditional
means).
They also include as special 
cases the rare-weak model~\citep{donoho2008higher,jin2009impossibility}
and a Boolean variant~\citep{helmbold2012necessity}.
We study the overparameterized regime,
when the dimension $p$ is significantly greater than the
number $n$ of samples. For a precise statement of our main result see
Theorem~\ref{t:main}.  After its statement, we give examples
of its consequences, including cases in which $s$ of the $p$
variables are relevant, but weakly associated with the
class designations.  
In some cases where $s$, $p$ and $n$
are polynomially related, the risk of the maximum-margin
algorithm approaches the Bayes-optimal risk as
$e^{-n^{\tau}}$, for $\tau > 0$.  
%



Analysis of classification is hindered by the fact that, in contrast to
regression, there is no known simple expression for the parameters as
a function of the training data.  Our analysis leverages recent
results, mentioned above, that characterize the weight vector obtained
by minimizing the logistic loss on training data
\citep{soudry2018implicit,ji2018risk,nacson2018stochastic}.  We use
this result not only to motivate the analysis of the maximum margin
algorithm, but also in our proofs, to get a handle on the relationship
of this solution to the training data.  When learning in the presence
of label noise, algorithms that minimize a convex loss face the hazard
that mislabeled examples can exert an outsized influence.  However, we show that in the over-parameterized regime this effect is
ameliorated.  In particular, we show that the ratio between the
(exponential) losses of any two examples is bounded above by an
absolute constant.  
One special case of our upper bounds is where
there are relatively few relevant variables, and many irrelevant
variables.  In this case, classification using only parameters that
correctly classify the clean examples with a large margin leads to
large loss on examples with noisy class labels.  
However, the training
process can use the parameters on irrelevant variables to play a role
akin to slack variables, allowing the algorithm to correctly classify
incorrectly labeled training examples with limited harm on independent
test data.  On the other hand, if there are too many
irrelevant variables, accurate classification is impossible
\citep{jin2009impossibility}.  Our bounds reflect this reality---if
the number of irrelevant variables increases while the number and
quality of the relevant variables remains fixed, ultimately 
our bounds degrade.

In simulation experiments, we see a decrease in population risk with
the increase of $p$ beyond $n$, as observed in previous double-descent
papers, but this is followed by an increase.  As mentioned above,
\citet{jin2009impossibility} showed that, under certain conditions, if
the number $p$ of attributes and the number $s$ of relevant attributes
satisfy $p \geq s^2$, then, in a sense, learning is impossible.  Our
experiments suggest that interpolation with logistic loss can succeed
close to this boundary, despite the lack of explicit regularization or
feature selection.

\subsection{Related Work} A number of recent papers have focused on bounding the asymptotic error of overparameterized decision rules. \citet{hastie2019surprises} and \citet{muthukumar2020harmless} studied the asymptotic squared error of the interpolating ordinary least squares estimator for the problem of overparameterized linear regression. This was followed by \citet{mei2019generalization} who characterize the asymptotic error of the OLS estimator in the random features model. As we do, \citet{montanari2019generalization} studied linear classification,
calculating a formula for the asymptotic test error of the maximum
margin classifier in the overparameterized setting when the features
are generated from a Gaussian distribution and the labels are
generated from a logistic link function. This was followed by the work of \citet{liang2020precise} who calculated a formula for the asymptotic test error of the maximum $\ell_1$-margin classifier in the same setting. \citet{DBLP:conf/icassp/DengKT20} independently obtained related
results, including analysis of the case where the marginal distribution over
the covariates is a mixture of Gaussians, one for each class.
Previously, \cite{candes2018phase,sur2019modern} studied the
asymptotic test error for this problem in the underparameterized
regime (when $p <n$).  In contrast with this previous work, 
we provide finite-sample bounds.

There has also been quite a bit of work on the non-asymptotic analysis
of interpolating estimators. \citet{liang2020just} provided a finite-sample upper bound the expected squared error for kernel
``ridgeless'' regressor, which interpolates the training data.
\citet{kobak2018optimal} provided an analysis of linear regression
that emphasized the role of irrelevant variables as providing
placeholders for learning parameters that play the role of slack
variables.  \citet{belkin2019two} provided a finite-sample analysis of
interpolating least-norm regression with feature selection.  They
showed that, beyond the point where the number of features included in
the model exceeds the number of training examples, the excess risk
decreases with the number of included features.  This analysis
considered the case that the covariates have a standard normal
distribution.  They also obtained similar results for a ``random
features'' model.  \citet{bartlett2019benign} provided non-asymptotic
upper and lower bounds on the squared error for the OLS estimator;
their analysis emphasized the effect of the covariance structure of
the independent variables on the success or failure of this estimator.
This earlier work studied regression; here we consider classification.
Study of regression is facilitated by the fact that the OLS parameter
vector has a simple closed-form expression as a function of the
training data.  
\citet{belkin2018overfitting} studied the
generalization error for a simplicial interpolating nearest neighbor
rule.  \citet{belkin2019does} provided bounds on the generalization
error for the Nadaraya-Watson regression estimator
applied to a singular kernel, a method that interpolates
the training data.  
\citet{liang2020MultipleDescent} provided upper bounds 
on the population risk for the least-norm interpolant
applied to a class of kernels including the Neural Tangent Kernel. 

In concurrent independent work, \citet{muthukumar2020classification}
studied the generalization properties of the maximum margin classifier
in the case where the marginal distribution on the covariates is
a single Gaussian, rather than a Gaussian per class.  They
showed that, in this setting, if there is enough overparameterization,
every example is a support vector, so that the 
maximum margin algorithm outputs the same parameters as the OLS
algorithm.  They also showed that the accuracy of the model,
measured using the 0-1 test loss, can be much better than its
accuracy with respect to the quadratic loss.

Additional related work is described in Section~\ref{s:related}.


\section{Definitions, Notation and Assumptions}
\label{s:defs}

Throughout
this section, $C > 0$ and $0 < \kappa < 1$ denote absolute constants.  
We will show that any choice $C$ that is large enough relative to
$1/\kappa$ will work.

We study learning from independent random examples
$(x,y) \in \R^{p} \times \{ -1, 1 \}$ sampled from a
joint distribution $\sfP$.  This distribution may be viewed
as a noisy variant of another distribution $\tsfP$
which we now describe.
A sample from $\tsfP$ may be generated by the following
process. 
\begin{enumerate}
\item  First, a clean label $\ty \in \{-1, 1\}$ 
is generated by flipping a fair coin.
\item Next,
$q \in \R^p$ is sampled from $\sfQ := \sfQ_1 \times \cdots \times \sfQ_p$, which
is an arbitrary
product distribution over $\R^p$ 
  \begin{itemize}
  \item whose marginals
are all zero-mean sub-Gaussians with sub-Gaussian norm at most $1$ (see Definition~\ref{def:subgaussian}), and
  \item such that $\mathbb{E}_{q\sim \sfQ}[\lv q \rv^2] \ge \kappa p$.
  \end{itemize}
\item For an arbitrary unitary 
matrix $U$
and $\mu \in \R^p$, $x = U q + \ty \mu$. This ensures that the mean of $x$ is $\mu$ when $\ty=1$ and is $-\mu$ when $\ty=-1$. 
\item Finally, 
noise is modeled as follows.
For $0 \leq \eta \leq 1/C$,
$\sfP$ is an arbitrary distribution over $\R^{p} \times \{ -1, 1\}$
\begin{itemize}
\item whose marginal distribution on $\R^{p}$ is 
the same as $\tsfP$, and 
\item such that
$d_{TV}(\sfP, \tsfP) \leq \eta$. 
\end{itemize}
Note that this definition includes the special case where $y$ is
obtained from $\ty$ by flipping it with probability $\eta$.
\end{enumerate}
Choosing a bound of $1$ on the sub-Gaussian norm 
of the components of $\sfQ$ fixes the scale of the data.
This simplifies the proofs without materially affecting the analysis, 
since rescaling the data does not affect the accuracy of the
maximum margin algorithm.

Let $(x_1,y_1),\ldots,(x_n,y_n)$ be $n$ training
examples drawn according to $\sfP$.  Let
\[
\Sample := \{ (x_1,y_1),\ldots,(x_n,y_n) \}.
\]
When $\Sample$ is linearly separable, 
let $w \in \R^{p}$ minimize $\lv w \rv$ subject to
$y_1 (w \cdot x_1) \geq 1,\ldots,y_n (w \cdot x_n) \geq 1$.
(In the setting that we analyze, we will show that, with high
probability, $\Sample$ is linearly separable.
When it is not, $w$ may be chosen arbitrarily.)

We will provide bounds on the misclassification probability of the classifier
parameterized by $w$ that can be achieved with probability
$1 - \delta$ over the draw of the samples.  

We make the following assumptions on the parameters of the problem:
\begin{enumerate}[label=(A.{\arabic*})]
\item the failure probability satisfies $0 \leq \delta < 1/C$,
\item number of samples satisfies $n \geq C \log(1/\delta)$,
\item the dimension satisfies $p \geq C \max\{\lv \mu\rv^2 n, n^2 \log (n/\delta) \}$,
\item the norm of the mean satisfies $\lv \mu \rv^2 \geq C \log(n/\delta)$.
\end{enumerate}

Here are some examples of generative models that fall within our framework.

\begin{example}[Gaussian class-conditional model]
The clean labels $\ty$ are drawn by flipping a fair coin. The distribution
on $x$, after conditioning on the value of $\ty$, is
$\sfN(\ty \mu, \Sigma)$, for $\Sigma$ with 
$\lv \Sigma \rv \leq 1$ and $\lv \Sigma^{-1} \rv \leq 1/\kappa$ (here $\lv \Sigma\rv$ is the matrix operator norm). 
\end{example}

\begin{example}[Noisy rare-weak model] A special case of the model described above is when $\Sigma =I$ and the mean vector $\mu$ is such that only $s$ components are non-zero and all 
non-zero entries are equal to $\gamma \in \mathbb{R}$.

\citet{donoho2008higher} studied this model in the noise-free
case (i.e.\ where $\eta = 0$).

\end{example}

\begin{example}[Boolean noisy rare-weak model]
Our assumptions are also satisfied\footnote{Strictly speaking,
$x$ needs to be scaled down to make the sub-Gaussian norm less than $1$ for this to be true, but this does
not affect the accuracy of the maximum margin classifier.}
by the following setting
with Boolean attributes. 
\begin{itemize}
\item $\ty \in \{ -1, 1\}$  is generated first, by flipping a fair coin.  
\item For $\gamma \in (0,1/2)$,
the components of $x \in \R^p$ are conditionally
independent given $\ty$: $x_1,\ldots,x_s$ are equal to $\ty$ with
probability $1/2 + \gamma$, $x_{s+1},\ldots,x_p$ are equal to
$\ty$ with probability $1/2$.  
\item $y$ is obtained from $\ty$ by 
flipping it with probability $\eta$.
\end{itemize}
The noiseless setting of this model was studied by \citet{helmbold2012necessity}.
\end{example}


\section{Main Result and Its Consequences}
Our main result is a finite-sample bound on the misclassification error of the maximum margin classifier.
\begin{theorem}
\label{t:main}
For all $0 < \kappa < 1$,
there is an absolute constant $c > 0$ such that, 
under the assumptions of Section~\ref{s:defs},
for all large enough $C$,
with probability $1 - \delta$, training on $\cS$
produces a maximum margin classifier $w$ satisfying
\[
\Pr_{(x,y) \sim \sfP} [\sign(w \cdot x) \neq y]
  \leq \eta + \exp\left(-c \frac{\lv \mu \rv^4}{p}\right).
\]
\end{theorem}

Consider the scenario where the number of samples $n$ is a constant,
but where the number of dimensions $p$ and $\lv \mu \rv$ are growing. Then our
assumptions require $\lv \mu \rv^2 = O(p)$.\footnote{
The definitions of ``big Oh notation'', i.e.\ 
$O(\cdot)$, $\omega(\cdot)$, $\Theta(\cdot),\Omega(\cdot)$, may be found
in \citep{cormen2009introduction}.}
But, for the misclassification
error to decrease we need $\lv \mu \rv^4 = \omega(p)$. Thus if, $\lv
\mu \rv = \Theta(p^{\beta})$ for any $\beta \in (1/4,1/2]$ then as
  $p\to \infty$, the misclassification error asymptotically will
  approach the noise level $\eta$.

Here are the implications of our results in the noisy rare-weak model. Recall that in this model $\mu$ is non-zero only on $s$ coordinates and the non-zero coordinates of $\mu$ are equal to some $\gamma$. Therefore, $\lv \mu \rv^2 = \gamma^2 s$.
\begin{corollary}\label{c:rareweak}
There is an absolute constant $c > 0$ such that, 
under the assumptions of Section~\ref{s:defs}, in the noisy rare-weak model,
for any $\gamma \geq 0$ 
and all large enough $C$, with probability $1 - \delta$, training on $\cS$
produces a maximum margin classifier $w$ satisfying
\[
\Pr_{(x,y) \sim \sfP} [\sign(w \cdot x) \neq y]
  \leq \eta + \exp\left(-c \frac{\gamma^4 s^2}{p}\right).
\]
\end{corollary}
Next, let us examine the implications of our results in the Boolean noisy rare-weak model. Here, $\lv \mu \rv^2 = 4\gamma^2 s$.
\begin{corollary}\label{c:boolean}
There is an absolute constant $c > 0$ such that, 
under the assumptions of Section~\ref{s:defs}, in the Boolean noisy rare-weak model 
for any $0<\gamma<1/2$ and all large enough $C$, with probability $1 - \delta$, training on $\cS$
produces a maximum margin classifier $w$ satisfying
\[
\Pr_{(x,y) \sim \sfP} [\sign(w \cdot x) \neq y]
  \leq \eta + \exp\left(-c \frac{\gamma^4 s^2}{p}\right).
\]
\end{corollary}
To gain some intuition let us explore the scaling of the misclassification error in these problems in different scaling limits for the parameters in both these problems.

Consider a case where, $\delta$, $\gamma$ and $n$ are 
constants and 
$s$ and $p$ grow.
Our assumptions hold
if $\lv \mu \rv^2 = \gamma^2 s =O(p)$. But for the misclassification
error to decrease we need $s^2 = \omega(p)$. So if $s =
\Theta(p^{\beta})$ where, $\beta \in (1/2, 1]$ then the
  misclassification error scales as $\eta + \exp(-c p^{2\beta-1})$ and
  asymptotically approaches $\eta$.

\citet{jin2009impossibility} showed that for the noiseless rare-weak model learning is impossible when $s = O(\sqrt{p})$ and $n$ is a constant. Our upper bounds show that, in a sense, the maximum margin classifier succeeds arbitrarily
close to this threshold.

Another interesting scenario 
is when $\delta$ and $\gamma$ are constants while both $s$ and $p$ grow as a function of the number of samples $n$. Let $p = \Theta(n^{2+\rho})$ and $s = \Theta(n^{1+\lambda})$, for positive $\rho$ and $\lambda$. Our assumptions are satisfied if $\rho >\lambda$ for large enough $n$, while, for the misclassification error to reduce with $n$ we need $2\lambda > \rho$. As $n$ gets larger the bound on the misclassification error scales as $\eta + \exp(-c n^{2\lambda-\rho})$ and gets arbitrarily close to $\eta$ for large enough $n$.
Informally, if the adversary fully expends its noise
budget, the Bayes error rate will be at least $\eta$; 
this is true in particular in the case where
labels are flipped with probability $\eta$.
In such cases, even if one could prove that the training data likely to be separated by
a large margin, 
the bound of Theorem~\ref{t:main}
approaches the Bayes error rate
faster than the standard margin bounds
\citep{vapnik1995nature,shawe1998structural}.  


\section{Proof of Theorem~\protect\ref{t:main}}\label{sec:proofsofthemainresult}
First, we may assume without loss of generality that
$U = I$.  To see this, note that 
\begin{itemize}
\item if $w$ is the maximum margin classifier for
$(x_1,y_1),\ldots,(x_n,y_n)$
then $U w$ is the 
maximum margin classifier for
$(U x_1,y_1),\ldots,(U  x_n,y_n)$, and
\item the probability that 
$y (w \cdot x) < 0$ is the same as the probability that
$y (U w \cdot U x) < 0$.
\end{itemize}
Let us assume from now on that $U = I$.

Our first lemma is an immediate consequence of the
coupling lemma \citep{lindvall2002lectures,Das11coupling} that allows us to handle the noise in the samples.
\begin{lemma}
\label{l:coupling}
There is a 
joint distribution on
$((x,y),(\tx,\ty))$ such that 
\begin{itemize}
\item the marginal on $(x,y)$ is $\sfP$,
\item the marginal on $(\tx,\ty)$ is $\tsfP$,
\item $\Pr[x = \tx] = 1$, and
\item $\Pr[y \neq \ty] \leq \eta$.
\end{itemize}
\end{lemma}

\begin{definition}
\label{d:N.C}
Let $(x_1,y_1,\ty_1),\ldots,(x_n,y_n,\ty_n)$ be $n$ i.i.d.\ draws from
the coupling of Lemma~\ref{l:coupling}, with the redundant
$\tx_1,\ldots,\tx_n$ thrown out.  Let $\cN$ 
be the set $\{ k : y_k \neq \ty_k \}$ 
of indices of ``noisy'' examples, and
$\cC = \{ k : y_k = \ty_k \}$ be the indices of
``clean'' examples.
\end{definition}

Note that Lemma~\ref{l:coupling} implies that
$(x_1,y_1),\ldots,(x_n,y_n)$ are $n$ i.i.d.\ draws from $\sfP$, as
before.

The next lemma is bound on the misclassification error in terms of the expected value of the \emph{margin} on clean points, $\mathbb{E}_{(x,\ty)\sim \sfP}[\ty(w\cdot x)] = \mu\cdot w$, and the norm of the classifier $w$.
\begin{lemma}
\label{l:dot_vs_norm}
There is an absolute positive constant 
$c$ such that
\[
\Pr_{(x,y) \sim \sfP} [ \sign(w \cdot x) \neq y]
 \leq \eta + \exp\left( -c \frac{(\mu \cdot w)^2}{\lv w \rv^2} \right).
\]
\end{lemma}
\begin{proof}Observe that
\begin{align*}
\Pr_{(x,y) \sim \sfP} [\sign(w \cdot x) \neq y] = \Pr_{(x,y) \sim \sfP} [y(w\cdot x) <0].
\end{align*}
For a draw $x,y,\ty$ from the coupling of Lemma~\ref{l:coupling},
we have
\begin{align*}
\Pr[y (w \cdot x) < 0]
& \leq \eta + \Pr[y (w \cdot x) < 0 \mbox{ and } y= \ty]\\
& = \eta + \Pr[\ty (w \cdot x) < 0].
\end{align*}
For $i \leq p$, the $i$th component of $\ty x$
is distributed as a $\mu_i + q_i$, where the $q_i\sim \sfQ_i$ is a mean zero random variable.  
Thus
$\E[\ty (w \cdot x)] = w \cdot \mu$,
so
\begin{align*}
\Pr [\ty (w \cdot x) < 0 ]
& = \Pr[\ty (w \cdot x)- \E[\ty (w \cdot x)]  < 
  - \mu \cdot w]\\
  &=\Pr[ w\cdot (\ty  x- \E[\ty  x])  < 
  - \mu \cdot w].
\end{align*}
An application of the general Hoeffding's inequality 
(see Theorem~\ref{thm:hoeffding}) upper bounds this probability and
completes the proof.
\end{proof}
In light of the previous lemma, next we prove a high probability lower bound on the expected margin on a clean point, $\mu\cdot w$.
\begin{lemma}
\label{l:dot.by.norm}
For all $0 < \kappa < 1$, there is an absolute positive constant $c$ such that,
for all large enough $C$,
with probability $1 - \delta$ over 
the random choice of $\cS$, it is linearly separable,
and the maximum margin weight vector $w$ satisfies,
\[
\mu \cdot w \geq \frac{ \lv w \rv \; \lv \mu \rv^2}{c \sqrt{p}}.
\]
\end{lemma}

Given these two main lemmas above, the main theorem follows immediately.
\begin{proof}(of Theorem~\ref{t:main}):  Combine the result of 
Lemma~\ref{l:dot_vs_norm} with the lower bound on $(\mu\cdot w)$ established in Lemma~\ref{l:dot.by.norm}.
\end{proof}

It remains to prove Lemma~\ref{l:dot.by.norm}, a lower bound on the expected margin on clean points $(\mu\cdot w)$. This crucial lemma is proved through a series of auxiliary lemmas,
which use a characterization of the maximum margin classifier $w$ in terms of iterates 
$\{v^{(t)} \}_{t=1}^{\infty}$ of gradient descent on the exponential loss. 
Denote
the risk associated with the exponential loss\footnote{We could also work with the logistic loss here, but the proofs are simpler if we work with the exponential loss without changing the conclusions.} as
\begin{align*}
R(v) := \sum_{k=1}^n \exp\left(-y_k v\cdot x_k\right).
\end{align*}
Then the iterates of gradient descent are defined as follows: \begin{itemize}
\item $\vzero := 0$, and
\item $\vtone := \vt -\alpha \nabla \Rt$, 
\end{itemize}
where $\alpha$ is a constant step-size.
\begin{lemma}[Soudry et al., 2018, Theorem 3]
\label{l:char.w}
For any linearly separable $\Sample$ and for all small enough step-sizes $\alpha$, 
we have
\[
\frac{w}{\lv w \rv} = \lim_{t \rightarrow \infty} \frac{\vt}{\lv \vt \rv}.
\]
\end{lemma}
\begin{definition}
For each index $k$ of an example, let $z_k := y_kx_k$.
\end{definition}
Most
of the argument required to prove Lemma~\ref{l:dot_vs_norm} is deterministic apart from some standard concentration arguments, which are gathered in the
following lemma.  (Recall that, since we are in the process
of proving Theorem~\ref{t:main}, the assumptions of
Section~\ref{s:defs} are in scope.)
\begin{lemma}\label{lem:normbounds} 
For all $\kappa > 0$, there is a $c \geq 1$ such that,
for all $c' > 0$, for all large
enough $C$,
with probability $1-\delta$ over the draw of the samples the following events simultaneously occur:
\begin{align}
&\text{For all }k \in [n],  \; \frac{p}{c} \le \lv z_k \rv^2 \le c p. \label{event:1}\\
&\text{For all }i\neq j \in [n], \;\lvert z_i \cdot z_j\rvert < c(\lv \mu \rv^2 + \sqrt{p \log(n/\delta)}).\label{event:2}\\
& \text{For all }k\in \mathcal{C}, \; \lvert \mu \cdot z_k -\lv \mu \rv^2\rvert < \lv \mu\rv^2/2.\label{event:3}\\
& \text{For all }k\in \mathcal{N}, \; \lvert \mu \cdot z_k -(-\lv \mu \rv^2)\rvert < \lv \mu\rv^2/2.\label{event:4}\\
&\text{The number of noisy samples satisfies } \lvert \mathcal{N} \rvert \le  (\eta + c') n. \label{event:5}\\
& \text{The samples are linearly separable.}
\end{align}
\end{lemma}
The proof of this lemma is 
in
Appendix~\ref{appA}. 

From here on, we will assume
that $\cS$ satisfies
all the conditions shown to hold with high probability
in Lemma~\ref{lem:normbounds}.

A concern is that, late in training, noisy examples will have outsized
effect on the classifier learned.  Lemma~\ref{l:loss.ratio} below
limits the
extent to which this can be true. It shows that throughout the
training process the loss on any one example is at most a constant
factor
larger than the loss on any other example. This is sufficient since
the gradient of the exponential loss
\begin{align*}
\nabla R(v) = -\sum_{k=1}^n  z_k \exp(-z_k \cdot v),
\end{align*}
is the sum of the $-z_k$ 
values weighted by their losses. 
We also know that with high probability $p/c\le\lv z_k\rv\le cp$, therefore, showing that the loss on a sample is within a constant factor of the loss of any other sample controls the influence that any one point can have on the learning process. We formalize this intuition in the proof of Lemma~\ref{l:dot.by.norm} in the sequel.

As will be clear in the proof of Lemma~\ref{l:loss.ratio}, 
the high dimensionality of the classifier ($p$ being larger than
$\lv\mu\rv^2 n$ and $n^2\log(n/\delta)$) is crucial in showing that
the ratio of the losses between any pair of points is bounded. Here is
some rough intuition why this is the case.

For the sake of intuition consider the extreme scenario where all the vectors $z_k$ are mutually orthogonal and $\lv z_i \rv = p$, for all $i\in [n]$. Then in this case, the change in the loss of a sample $i \in [n]$ due to each gradient descent update will be independent of  any other sample $j\neq i \in [n]$ and all the losses will decrease exactly at the same rate. 
Lemma~\ref{lem:normbounds} implies that, when $p$ is large enough
relative to $\lv \mu \rv$, the $z_k$ vectors are nearly
pairwise orthogonal.  In this case, the losses remain within
a constant factor of one another.
%
\begin{lemma}
\label{l:loss.ratio}
There is an absolute constant $c$ such that, for all large enough $C$,
and all small enough step sizes $\alpha$, for all iterations $t \ge
0$,
\begin{align*}
\Atmax := \max_{k,\ell \in \Sample} \left\{\frac{\exp(-\vt \cdot z_k)}{\exp(-\vt \cdot z_{\ell})} \right\}
 \leq c.
\end{align*}
\end{lemma}

\begin{proof} 
$\Atmax$ is the maximum ratio between a pair of samples at iteration $t$. 
Let $c_1$ be the constant $c\ge 1$ from Lemma~\ref{lem:normbounds}.
We will prove that 
$\Atmax \leq 4 c_1^2$  for all $t \geq 0$ 
by using an inductive argument over the iterations $t$. 

Let us begin by establishing this for the base case, when $t=0$. Since the gradient descent algorithm is initialized at the origin, the loss for any sample $j \in [n]$ is $\exp(-0\cdot z_j) = 1$. Therefore, $A_0^{\mathrm{max}} = 1 < 4c_1^2$.

Assume that the inductive hypothesis holds for some iteration $t$, 
we shall now prove that then it must also hold at iteration $t+1$. 

To simplify notation we shall analyze the ratio between the losses on
the first and the second sample but a similar analysis holds for any
distinct pair. Let $\Gt$ be the loss on sample $z_1 $ and let $\Ht$ be
the loss on sample $z_2$ at the $t^{th}$ iteration. Define $\At :=
\Gt/\Ht$ to be the ratio of the losses at iteration $t$.

By the definition of $\vtone$ as the gradient descent iterate
\begin{align*}
\Atone &= \frac{\exp\left(-\vtone\cdot z_1\right)}{\exp\left(-\vtone\cdot z_2\right)}\\
& = \frac{\exp\left(-(\vt - \alpha \nabla \Rt)\cdot z_1\right)}{\exp\left(-(\vt - \alpha \nabla \Rt)\cdot z_2\right)}\\ &= \frac{\exp\left(-\vt\cdot z_1\right)}{\exp\left(-\vt\cdot z_2\right)} \cdot \frac{\exp(\alpha \nabla \Rt \cdot z_1)}{\exp(\alpha \nabla \Rt \cdot z_2)}\\
&= \At \cdot \frac{\exp(-\alpha \sum_{j \in [n]} z_j\cdot z_1\exp(-\vt \cdot z_j))}{\exp(-\alpha \sum_{j \in [n]} z_j\cdot z_2\exp(-\vt \cdot z_j))}\\
& = \At \cdot \frac{\exp(-\alpha \lv z_1\rv^2\Gt)}{\exp(-\alpha \lv z_2\rv^2 \Ht)}\frac{\exp(-\alpha \sum_{j > 1} z_j\cdot z_1\exp(-\vt \cdot z_j))}{\exp(-\alpha \sum_{j \neq 2} z_j\cdot z_2\exp(-\vt \cdot z_j))}.
\end{align*}
Recalling that $c_1$ is the constant $c$ from Lemma~\ref{lem:normbounds},
by \eqref{event:1}, we have
\begin{align*}
 \frac{p}{c_1} \le \lv z_i \rv^2 \le c_1 p, \quad \text{ for all }i \in [n],
\end{align*}
and \eqref{event:2} gives
\begin{align*}
\lvert z_i \cdot z_j \rvert < c_1 (\lv \mu\rv^2 + \sqrt{p \log(n/\delta)}), \quad \text{ for all } i\neq j \in [n].
\end{align*}
These, combined with the the expression for $\Atone$ above, give
\begin{align*}
&\Atone\\ &= \At \cdot \exp(-\alpha \lv z_1\rv^2\Gt + \alpha \lv z_2\rv^2 \Ht ) \cdot\frac{\exp(-\alpha \sum_{j > 1} z_j\cdot z_1\exp(-\vt \cdot z_j))}{\exp(-\alpha \sum_{j \neq 2} z_j\cdot z_2\exp(-\vt \cdot z_j))}\\
& \le \At  \exp\left(-\alpha p \left(\frac{\Gt}{c_1} - c_1 \Ht\right)\right)
\exp\left(2\alpha c_1 (\lv \mu\rv^2 +\sqrt{p \log(n/\delta)})\sum_{j \in [n]}\exp(-\vt\cdot z_j)\right)
\end{align*}
which implies
\begin{align} \label{eq:upperbound_atone}
\nonumber&\Atone\\ & \leq \At
 \exp\left(-\frac{\alpha H_t p}{c_1} \left(A_t - c_1^2\right) \right)
 \exp\left(
2\alpha c_1 (\lv \mu\rv^2 +\sqrt{p \log(n/\delta)})
    \sum_{j \in [n]}\exp(-\vt\cdot z_j)\right).
\end{align}
Consider two disjoint cases. 

\noindent
\textbf{Case 1 ($\At \le 2 c_1^2$):}   
Using inequality~\eqref{eq:upperbound_atone}
\begin{align*}
\Atone &\le \At \exp\left(-\frac{\alpha  \Ht p }{c_1} (\At - c_1^2)\right) \exp\left(
2\alpha c_1 (\lv \mu\rv^2 +\sqrt{p \log(n/\delta)})
   \sum_{j \in [n]}\exp(-\vt\cdot z_j)\right)\\
& \le \At \exp\left(c_1 \alpha  \Ht p\right) \exp\left(
2\alpha c_1 (\lv \mu\rv^2 +\sqrt{p \log(n/\delta)})
\sum_{j \in [n]}\exp(-\vt\cdot z_j)\right)\\
& \overset{(i)}{\le} \At \exp\left(c_1 \alpha p n\right) \exp\left(
2\alpha c_1 (\lv \mu\rv^2 +\sqrt{p \log(n/\delta)})
n\right) \\
&= \At \exp\left(\alpha (c_1 p+
2 c_1 (\lv \mu\rv^2 +\sqrt{p \log(n/\delta)})
)n \right)\\
& \overset{(ii)}{\le} 2 c_1^2 \exp(1/8) < 4 c_1^2
\end{align*}
where $(i)$ follows since the sum of the losses on all samples is always smaller than the initial loss which is $n$ 
(see Lemma~\ref{l:n.bound})
and $\Ht \le n$, while, $(ii)$ follows as the step-size may be chosen to be at most 
$(8 c_1 (p+
2 (\lv \mu\rv^2 +\sqrt{p \log(n/\delta)})
 \lv \mu \rv^2)n)^{-1}$.

\noindent
\textbf{Case 2 ($\At > 2 c_1^2$) :}  
Reusing inequality~\eqref{eq:upperbound_atone},
\begin{align*}
\Atone & \leq \At
 \exp\left(-\frac{\alpha H_t p}{c_1} \left(A_t - c_1^2\right) \right)
\exp\left(
2\alpha c_1 (\lv \mu\rv^2 +\sqrt{p \log(n/\delta)})
\sum_{j \in [n]}\exp(-\vt\cdot z_j)\right) \\
 & = \At
 \exp\left(-\frac{\alpha H_t p}{c_1} \left(A_t \!-\! c_1^2\right) \right)
\exp\left(
2\alpha c_1 (\lv \mu\rv^2 \!+\!\sqrt{p \log(n/\delta)})
H_t \sum_{j \in [n]} \frac{\exp(-\vt\cdot z_j)}{H_t} \right) \\
 & \leq \At
 \exp\left(-\frac{\alpha H_t p}{c_1} \left(A_t - c_1^2\right) \right)
\exp\left(8 \alpha c_1^3 (\lv \mu\rv^2 +\sqrt{p \log(n/\delta)}) H_t n \right) 
  \;\;\;\;\mbox{(by the IH)} \\
 & = \At \exp\left( - \alpha H_t
      \left( \frac{p}{c_1} \left(A_t - c_1^2\right) 
         - 8 c_1^3 (\lv \mu\rv^2 +\sqrt{p \log(n/\delta)}) n \right) \right)
             \\
 & \leq \At \exp\left( - \alpha H_t
      \left( c_1 p
         - 8 c_1^3 (\lv \mu\rv^2 +\sqrt{p \log(n/\delta)}) n \right) \right).
\end{align*}
Since $p > C \lv\mu \rv^2$ and $p > C n^2 \log (n/\delta)$,
and noting that Lemma~\ref{lem:normbounds} is consistent with
$C$ being arbitrarily large while $c_1$ remains fixed, we have
that, in this case, $\Atone \leq \At$ (as the term in the exponent is non-positive).  This completes the proof
of the inductive step in this case, and therefore the entire proof.

\end{proof}

\subsection{Proof of Lemma~\ref{l:dot.by.norm}} 
Armed with Lemma~\ref{l:loss.ratio}, we now prove
Lemma~\ref{l:dot.by.norm}.

Let us proceed assuming that the event defined in Lemma~\ref{lem:normbounds} occurs, and, in this proof, 
let $c_1$ be the constant $c$ from that lemma. We know that this event occurs with probability at least $1-\delta$.

We have
\begin{align*}
\mu \cdot v^{(t+1)} = \mu \cdot v^{(t)} + \alpha\sum_{k=1}^n (\mu \cdot z_k) \exp\left( -v^{(t)}\cdot z_k\right). 
\end{align*}
Dividing the sum into the clean and noisy examples, we have
\begin{align}
\label{e:clean.noisy}
 \mu \cdot v^{(t+1)} 
& = \mu \cdot v^{(t)}
  + \alpha\sum_{k \in \cC} \left( \mu \cdot z_k \right) \exp\left( -v^{(t)}\cdot z_k\right)\\
  & \qquad \qquad \qquad \qquad \qquad \qquad\nonumber+\alpha\sum_{k \in \cN} \left( \mu \cdot z_k \right) \exp\left( -v^{(t)}\cdot z_k\right). 
\end{align}
Combining~\eqref{event:3}, \eqref{event:4} and \eqref{e:clean.noisy} we infer
\begin{align}
\nonumber
\mu \cdot v^{(t+1)} 
&\ge \mu \cdot v^{(t)} + \frac{\lv \mu \rv^2 \alpha }{2}\sum_{k \in \cC} \exp\left( -v^{(t)}\cdot z_k\right)
 - \frac{3  \lv \mu \rv^2 \alpha }{2}\sum_{k \in \cN} \exp\left( -v^{(t)}\cdot z_k\right) \\ 
\label{e:by.noisy.loss}
&= \mu \cdot v^{(t)} + \frac{\lv \mu \rv^2}{2} \alpha R(v^{(t)})
 -2  \lv \mu \rv^2 \alpha\sum_{k \in \cN} \exp\left( -v^{(t)}\cdot z_k\right).
\end{align}
Since $|\cN| \leq (\eta + c') n$, where $c'$ is an arbitrarily small constant,
if $c_2$ is the constant from Lemma~\ref{l:loss.ratio}, we have
\[
\sum_{k \in \cN} \exp\left( -v^{(t)}\cdot z_k\right) \leq 
c_2 (\eta + c') n \min_k \exp\left( -v^{(t)}\cdot z_k\right)
\leq c_2 (\eta + c') R(v^{(t)})
\leq  R(v^{(t)})/4,
\]
since $\eta \leq 1/C$.
Thus inequality~\eqref{e:by.noisy.loss} implies
\[
\mu \cdot v^{(t+1)} \geq \mu \cdot v^{(t)} + \frac{\lv \mu \rv^2 \alpha }{4 } R(v^{(t)}).
\]
Unrolling this via an induction yields
\[
\mu \cdot v^{(t+1)}
  \geq \frac{\alpha \lv \mu \rv^2}{4} \sum_{m=0}^t R(v^{(m)}) \qquad \text{(since $v^{(0)}=0$).}
\]
Now let us multiply both sides by $\lv w\rv/\lv v_{t+1}\rv$
\[
\lv w \rv \frac{\mu \cdot v^{(t+1)}}{\lv v^{(t+1)} \rv}
  \geq \lv w \rv \frac{\alpha \lv \mu \rv^2 \sum_{m=0}^t R(v^{(m)})}{4 \lv v^{(t+1)} \rv}.
\]
Next, let us take the large-$t$ limit.  Applying Lemma~\ref{l:char.w} to
the left hand side, 
\begin{equation}
\label{e:by.grad.sum}
\mu \cdot w \geq 
 \alpha \lv w \rv \lv \mu \rv^2  
  \lim_{t \rightarrow \infty} \frac{\sum_{m=0}^t R(v^{(m)})}
                                {4 \lv v^{(t+1)} \rv}.
\end{equation}
By definition of the gradient descent iterates
\begin{align*}
\lv v^{(t+1)} \rv
 & = \left\lv \sum_{m=0}^t \alpha \nabla R(v^{(m)}) \right\rv \\
 & \leq \alpha \sum_{m=0}^t \lv \nabla R(v^{(m)}) \rv \\
 & = \alpha \sum_{m=0}^t 
        \left\lv \sum_{k=1}^n z_k \exp(-v^{(m)} \cdot z_k) \right\rv \\
 & \leq \alpha\sum_{m=0}^t 
        \sum_{k=1}^n \exp(-v^{(m)} \cdot z_k) \lv z_k \rv \\
 & \leq \alpha c_1 \sqrt{p} 
   \sum_{m=0}^t R(v^{(m)}).
\end{align*}
This together with inequality~\eqref{e:by.grad.sum} yields
\[
\mu \cdot w \geq  \frac{\lv w \rv \lv \mu \rv^2}{4 c_1 \sqrt{p}},
\]
completing the proof.

\section{Simulations} \label{s:simulations}We experimentally study the behavior of the maximum margin classifier in the overparameterized regime on synthetic data generated according to the Boolean noisy rare-weak model. Recall that this is a model where the clean label $\ty \in \{-1,1\}$ is first generated by flipping a fair coin. Then the covariate $x$ is drawn from a distribution conditioned on $\ty$ such that $s$ of the coordinates of $x$ are equal to $\ty$ with probability $1/2+\gamma$ and the other $p-s$ coordinates are random and independent of the true label. The noisy label $y$ is obtained by flipping $\ty$ with probability $\eta$. In this section the flipping probability $\eta$ is always $0.1$. In all our experiments the number of samples $n$ is kept constant at $100$. 

In the first experiment in Figure~\ref{fig:testerrorwithp} we hold $n$
and the number of relevant attributes $s$ constant and vary the
dimension $p$ for different values of $\gamma$. We find that after an
initial dip in the test error (for $\gamma = 0.2,0.3$) the test error
starts to rise slowly with $p$, as in our upper bounds.

Next, in Figure~\ref{fig:testerrorwiths} we explore the scaling of the test error with the number of relevant attributes $s$ when $n$ and $p$ are held constant. As we would expect, the test error decreases as $s$ grows for all the different values of $\gamma$.

Finally, in Figure~\ref{fig:threshold} we study how the test error
changes when both $p$ and $s$ are increasing when $n$ and $\gamma$ are
held constant. Our results (see~Corollary~\ref{c:boolean}) 
do not guarantee learning when $s = \Theta(\sqrt{p})$ 
(and \citet{jin2009impossibility} proved that learning
is impossible in a related setting, even in the absence of noise);
we
find that the test error remains constant in our experiment in this
setting.   In the cases when $s=p^{0.55}$ and when $s=p^{0.65}$,
slightly beyond this threshold,
the test error approaches the Bayes-optimal error as $p$ gets large in
our experiment. This provides 
experimental evidence that 
the maximum margin algorithm, without explicit regularization or feature
selection, even in the presence of noise, learns with using
a number of relevant variables near the theoretical limit of
what is possible for any algorithm.
\citep[Note that, as emphasized by][the
fraction of relevant variables is going to zero as $p$ increases
in these experiments.]{helmbold2012necessity}

\begin{figure}[t]
\centering
\includegraphics[scale=0.7]{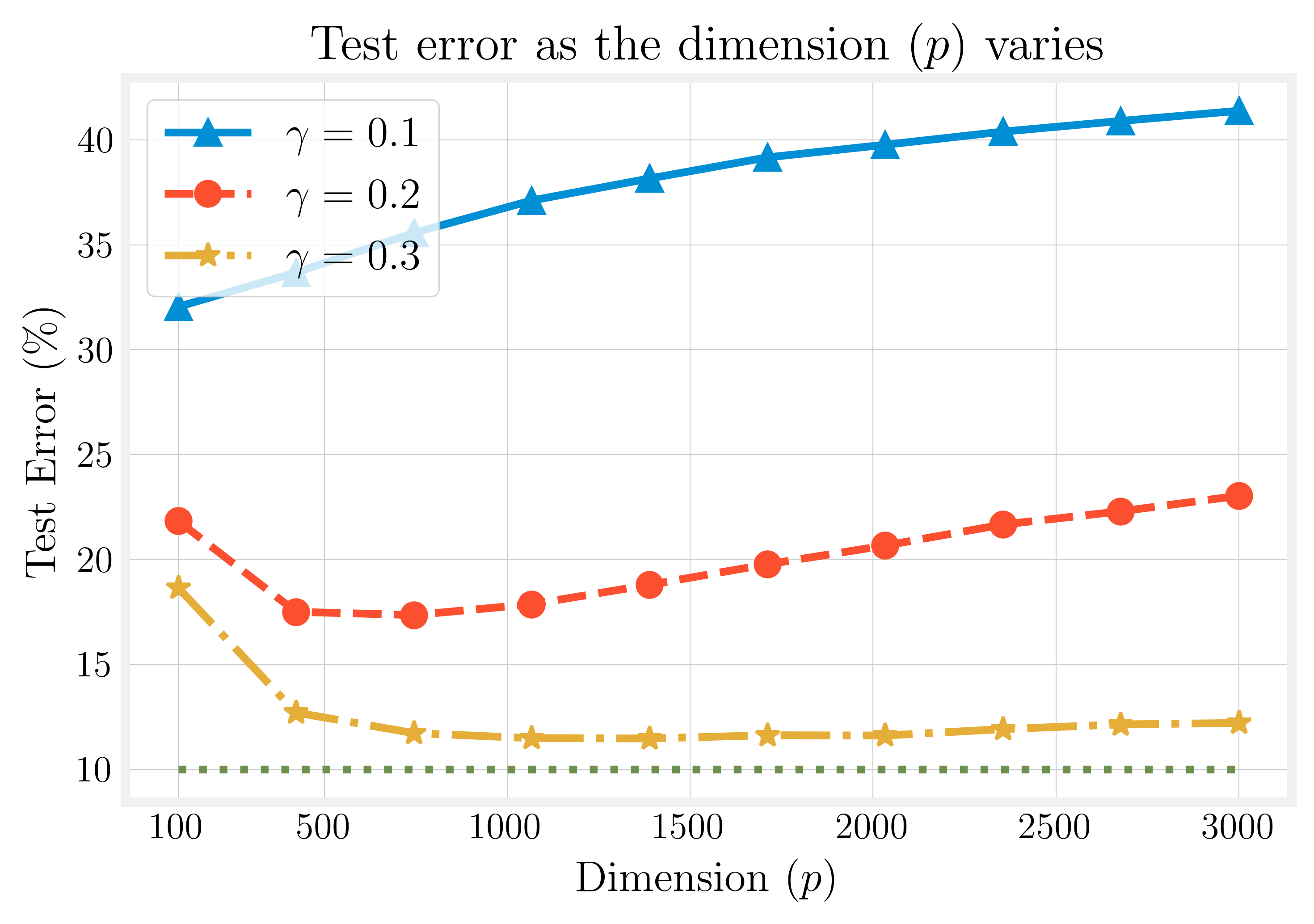}
\caption{\label{fig:testerrorwithp}Plot of the test error versus the dimension of the covariates $p$ for different values of $\gamma$. The number of samples $n=100$ and the number of relevant variables $s=50$ are both kept fixed. The dimension $p$ is varied in the interval $[100,3000]$. The data is generated according to the Boolean noisy rare-weak model. The dotted olive green line represents the noise level ($10\%$). The plot is generated by averaging over $500$ draws of the samples. The train error on all runs was always $0$.}
\end{figure}

\begin{figure}[h]
\centering
\includegraphics[scale=0.7]{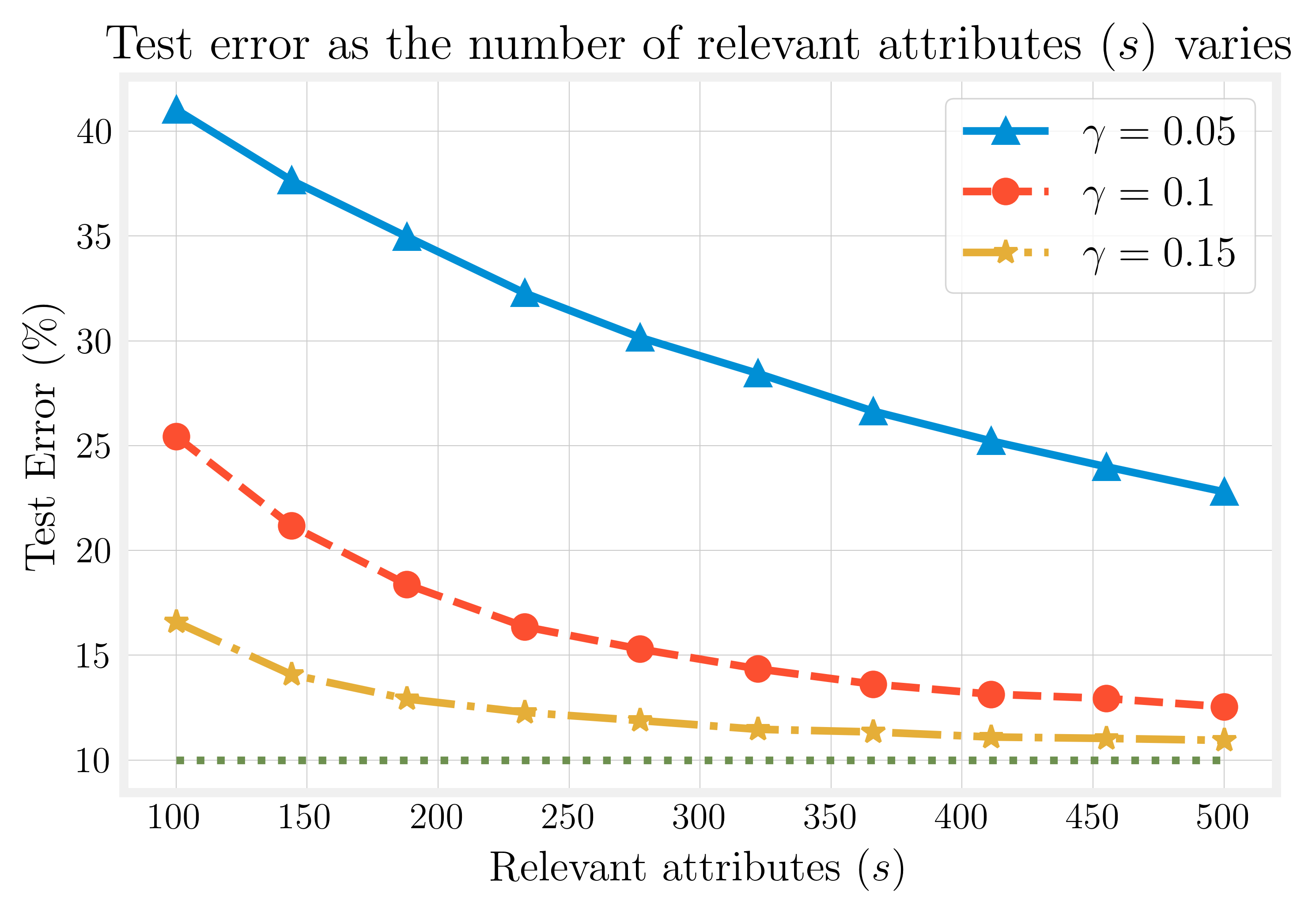}
\caption{\label{fig:testerrorwiths}Plot of the test error versus the number of relevant attributes $s$ for different values of $\gamma$. The number of samples $n=100$ and the dimension $p=500$ are both kept fixed. The dimension $s$ is varied in the interval $[100,500]$. The data is generated according to the Boolean noisy rare-weak model. The dotted olive green line represents the noise level ($10\%$). The plot is generated by averaging over $500$ draws of the samples. The train error on all runs was always $0$.}
\end{figure}

\begin{figure}[h]
\centering
\includegraphics[scale=0.7]{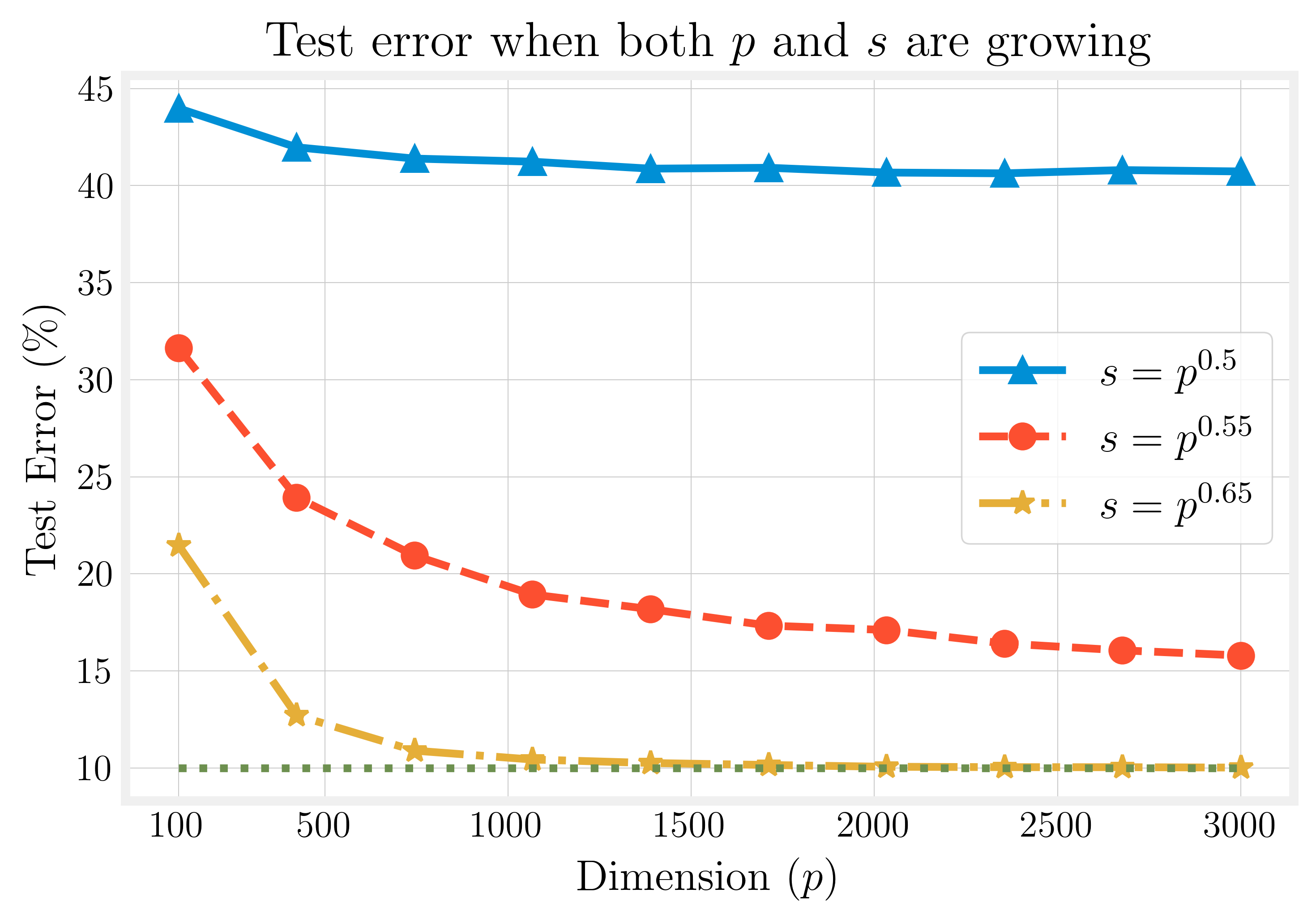}
\caption{\label{fig:threshold}Plot of the test error versus the dimension (p) for different scalings of $s$ with $p$. The number of samples $n=100$ and $\gamma=0.1$ are both held fixed. The dimension $p$ is varied in the interval $[100,3000]$. The data is generated according to the Boolean noisy rare-weak model. The dotted olive green line represents the noise level ($10\%$). The plot is generated by averaging over $500$ draws of the samples. The train error on all runs was always $0$.}
\end{figure}

\section{Additional Related Work}
\label{s:related}

\citet{ng2002discriminative} compared the Naive Bayes algorithm, which
builds a classifier from estimates of class-conditional distributions
using conditional independence assumptions, with discriminative
training of a logistic regressor.  Their main point is that Naive
Bayes converges faster.  Our analysis provides a counterpoint to
theirs, showing that, for a reasonable data distribution that includes
label noise, in the overparameterized regime, unregularized
discriminative training with a commonly used loss function learns a
highly accurate classifier from a constant number of examples.

Analysis of learning with class-conditional Gaussians with the same
covariance structure has been extensively studied in the case that the
number $n$ of training examples is greater than the number $p$ of
parameters \citep[see][]{anderson2003introduction}.  When $p \gg n$,
\citet{bickel2004some} showed that, even when the class-conditional
distributions do not have diagonal covariance matrices, 
behaving as if they are can lead to
improved classification accuracy.  The model that we use is a
generalization of the rare-weak model studied by
\citet{donoho2008higher} (see Section~\ref{s:defs} for details of our
set-up).  The class-conditional distributions studied there have a
standard multivariate normal distribution, while our results hold for
a more general class of sub-Gaussian class-conditional distributions.
More importantly, in order to address the experimental findings of
\citet{zhang2016understanding}, we have have also supplemented the
rare-weak model to include label noise.  Finite-sample bounds for
algorithms using $\ell_1$ penalties, again, in the absence of label
noise have been obtained~\citep{cai2011direct,li2015fast,li2017l1,tony2019high}.
\citet{dobriban2018high} studied regularized classification in the
asymptotic framework where $p$ and $n$ go to infinity together.
\citet{fan2008high} and \citet{jin2009impossibility} proved that
learning with class-conditional Gaussians is impossible when too few
variables are associated with the class designations.  Our analysis
shows that, even in the presence of misclassification noise, in a
sense, the maximum-margin algorithm succeeds up to the edge of the
frontier established by one of the results by \citet{jin2009impossibility}.  \citet{nagarajan2019uniform} used a
data distribution like the distributions considered here, but to
analyze limitations of uniform convergence tools.

The framework studied here also includes as a special case 
the setting studied by
\citet{helmbold2012necessity}, 
with Boolean attributes; again, a key
modification is the addition of misclassification noise.
Also, while the upper bounds of \citet{helmbold2012necessity} 
are for algorithms that perform unweighted 
votes
over selected attributes, here we consider the maximum margin algorithm.
A more refined analysis of learning with conditionally independent
Boolean attributes was carried out by \citet{berend2015finite}.
\citet{kleindessner2018crowdsourcing} studied learning with
conditionally independent Boolean attributes in the presence
of noise---they analyzed tasks other than classification, including
estimating the degree of association between the attributes
(viewed in that work as experts) and the true class designations.

As mentioned above, we consider the case that the data is 
corrupted with label noise.  We 
consider
adversarial label noise
\citep{kearns1994toward,kalai2008agnostically,
klivans2009learning,awasthi2017power,DBLP:conf/nips/Talwar20}.  
In this model,
an adversary is allowed to change the classifications of an arbitrary subset
of the domain whose probability is $\eta$, while leaving the
marginal distribution on the covariates unchanged.  It includes
as a special case the 
heavily studied situation in which classifications are
randomly flipped with probability $\eta$
\citep{angluin1988learning,kearns1998efficient,cesa1999sample,servedio1999pac,kalai2005boosting,long2010random,van2015learning}
along with variants that allow limited dependence
of the probability that a label is corrupted
on the clean example
\citep{lugosi1992learning,natarajan2013learning,scott2013classification,cannings2020classification}.
Adversarial label noise
allows for the possibility that noise is concentrated in
a part of the domain, where noisy examples 
have greater potential to coordinate their
effects; it is a weaker
assumption than even Massart noise 
\citep{massart2006risk,blanchard2008statistical,awasthi2015efficient,diakonikolas2019distribution},
which requires a separate limit 
on the conditional probability of an incorrect
label, given any clean example.
We show that, with sufficient overparameterization,
even in the absence of regularity in the noise,
the algorithm that simply
minimizes the standard
softmax loss without any explicit regularization
enjoys surprisingly strong noise tolerance.

After a preliminary version of this paper was posted on arXiv~\citep{chatterji2020finite}, some related work was published
\citep{wang2020benign,hsu2021proliferation,liang2021interpolating}.


\section{Discussion} 

Even in the presence of misclassification noise, 
with sufficient overparameterization, unregularized minimization
of the logistic loss produces accurate classifiers when the
clean data has well-separated sub-Gaussian class-conditional
distributions.

We have analyzed the case of a linear classifier without a bias term.
In the setting studied here, the Bayes-optimal classifier has a bias
term of zero, and adding analysis of a bias term in the maximum margin
classifier would complicate the analysis without significantly
changing the results.

In the noisy rare-weak model, when $p$ and $s$ scale favorably with
$n$, and $\gamma$ is a constant, the excess risk of the maximum margin
algorithm decreases very rapidly with $n$.
One contributing cause is a ``wisdom of the crowds'' effect that is
present when classifying with conditionally independent attributes: a
classifier can be very accurate, even when the angle between its
normal vector and the optimum is not very small.  For example, if
$100$ experts each predict a binary class, and they are correct
independently with probability $3/4$, a vote over their predictions
remains over 95\% accurate even if we flip the votes of $25$ of them.
(Note that, even in some cases where Lemma~\ref{l:dot.by.norm} implies
accuracy very close to optimal, it may not imply that the cosine of the
angle between $\mu$ and $w$ is anywhere near $1$.)  On the other hand,
the concentration required for successful generalization is robust to
departures from the conditional independence assumption.  Our
assumptions already allow substantial class-conditional dependence
among the attributes, but it may be interesting to explore even weaker
assumptions.

We note that a bound on the accuracy of the maximum margin classifier
with respect to the distribution $\tsfP$ without any label noise is
implicit in our analysis.  (The bound is the same
as Theorem~\ref{t:main}, but without the $\eta$.)

Our bounds show that the maximum margin classifier approaches
the Bayes risk as the parameters go to infinity in various ways.
It would be interesting to characterize the conditions under which
this happens.  A related question is to prove lower bounds in
terms of the parameters of the problem.  Another is to prove
bounds for finite $p$ and $n$ under weaker conditions.

We assume that the distributions of $x - \mu$ and $x - (-\mu)$
are the same.  This is useful in particular for simplifying the
analysis of dot products between examples of opposite classes.
It should not be difficult to extend the analysis meaningfully
to remove this assumption---we use this assumption in this paper
to keep the analysis as simple and clean as possible.

The distribution of $x - \tilde{y}\mu$ comes from a sub-Gaussian distribution
obtained by applying a unitary transformation to a latent variable
with a product distribution.  Concentration theorems have
been proved under many conditions weaker than independence
\citep[see][]{schmidt1995chernoff,dubhashi1998balls,pemmaraju2001equitable}.
Our analysis can be straightforwardly extended using these 
weaker assumptions.

The assumption that the latent variable $q$
satisfies $\E_{q \sim \mathsf{Q}} [\lv q \rv^2] \geq \kappa p$
formalizes the notion that, in a sense,
these variables are truly used, which is required for
concentration.  We believe that smaller values of 
$\E_{q \sim \mathsf{Q}} [\lv q \rv^2]$ are compatible with
successful learning.  We chose this assumption to facilitate a
simple and clean analysis, but an analysis that separates
the dependence on $\E_{q \sim \mathsf{Q}} [\lv q \rv^2]$ is a potential
subject for future work.

The lower bounds on $p$ are needed for concentration, as described
earlier.  We suspect that the requirement that
$p = \Omega(n^2 \log(n/\delta))$ can be improved.  
The bottleneck is in the proof of Lemma~\ref{l:loss.ratio}.
(As we mentioned earlier, a larger value of $p$ promotes
the property that the loss on an example can be reduced
by gradient descent without increasing the loss on other examples very much.)

The lower bound on $\lv \mu \rv^2$ allows us to focus on the
case where most clean examples are classified correctly
by a large margin, which is the case that we want to focus
on.  This could potentially be weakened or removed through
a case analysis, exploiting the fact that a weaker bound
is needed in the case that $\lv \mu \rv$ is small.

Using the generalized Hoeffding bound (Theorem~\ref{thm:hoeffding})
it is not hard to show 
\citep[see][Theorem 1]{helmbold2012necessity}
that, in our setting, the
Bayes optimal classifier has error at most
$\eta + \exp(- c\lv \mu \rv^2)$ for an absolute
constant $c$, and Slud's Lemma gives a similar
lower bound 
\citep[see][]{anthony2009neural}
\citep[see also][inequality~(8)]{helmbold2012necessity}.
Our upper bound of 
$\eta + \exp(- c'\lv \mu \rv^4/p)$ 
for the maximum margin algorithm applied to finite training data
is worse than
this by a factor of $\lv \mu\rv^2/p$ in the exponent.  

Implicit regularization lemmas like the one that was so helpful to us have been
obtained for other problems
\citep[see][]{gunasekar2017implicit,gunasekar2018implicit,woodworth2020kernel,azulay2021implicit}.  We hope
that further
advances in implicit regularization research could be combined with
the techniques of this paper to prove generalization guarantees
for 
interpolating classifiers using richer model classes, including
neural networks.

\subsection*{Acknowledgements} 
We thank anonymous reviewers for their valuable comments
and suggestions. NC gratefully acknowledges the support of the NSF through grants IIS-1619362 and IIS-1909365.

\appendix

\section{Concentration Inequalities}\label{appA}
In this section we begin by presenting a definition of sub-Gaussian and sub-exponential random variables in terms of Orlicz norms. Then we state a version of Hoeffding's inequality and a version of Bernstein's inequality. Finally, we prove Lemma~\ref{lem:normbounds} which implies that a good event which our proofs rely on holds with high probability.

For an excellent reference of sub-Gaussian and sub-exponential concentration inequalities we refer the reader to \citet[Chapter~2]{vershynin2018high}.

\begin{definition}[sub-Gaussian random variable] \label{def:subgaussian}A random variable $\theta$ is sub-Gaussian if 
\begin{align*}
\lv \theta \rv_{\psi_2}:= \inf\left\{t>0: \mathbb{E}[\exp(\theta^2/t^2)]< 2\right\}
\end{align*}
is bounded. Further, $\lv \theta\rv_{\psi_2}$ is defined to be its sub-Gaussian norm.
\end{definition}

We now state general Hoeffding's inequality \citep[see][Theorem~2.6.3]{vershynin2018high} a concentration inequality for a weighted sum of independent sub-Gaussian random variables.
\begin{theorem}[General Hoeffding's inequality] \label{thm:hoeffding}Let $\theta_1,\ldots,\theta_m$ be independent mean-zero sub-Gaussian random variables and $a=(a_1,\ldots,a_m) \in \mathbb{R}^m$. Then, for every $t>0$, we have
\begin{align*}
\Pr\left[\Big\lvert \sum_{i=1}^m a_i \theta_i \Big\rvert \ge t\right]\le 2\exp\left(-\frac{c_2 t^2}{K^2 \lv a\rv^2}\right),
\end{align*}
where $K = \max_i \lv \theta_i \rv_{\psi_2}$ and
$c_2$ is an absolute constant.
\end{theorem}
A one-sided version of this theorem (upper/lower deviation bound) holds without the factor of $2$ multiplying the exponent on the right hand side.

\begin{definition}[sub-exponential random variable] \label{def:subexp}A random variable $\theta$ is 
said to be
sub-exponential if 
\begin{align*}
\lv \theta \rv_{\psi_1}:= \inf\left\{t>0: \mathbb{E}[\exp(\lvert \theta\rvert/t)< 2]\right\}
\end{align*}
is bounded. Further, $\lv \theta\rv_{\psi_1}$ is defined to be its sub-exponential norm.
\end{definition}

We shall also use Bernstein's inequality \citep[see][Theorem~2.8.1]{vershynin2018high} a concentration inequality for a sum of independent sub-exponential random variables.

\begin{theorem}[Bernstein's inequality] \label{thm:bernstein}
For independent mean-zero sub-exponential random variables $\theta_1,\ldots,\theta_m$, for every $t>0$, we have
\begin{align*}
\Pr\left[\Big\lvert \sum_{i=1}^m  \theta_i \Big\rvert \ge t\right]\le 2\exp\left(-c_1 \min\left\{\frac{t^2}{\sum_{i=1}^m \lv \theta_i\rv_{\psi_1}^2},\frac{t}{\max_i \lv \theta_i \rv_{\psi_1}}\right\}\right),
\end{align*}
where $c_1$ is
an absolute constant.
\end{theorem}
Again note that a one-sided version of this inequality holds without the factor of $2$ multiplying the exponent on the right hand side.

%
%
%

We break the proof of Lemma~\ref{lem:normbounds} into different parts,
which are proved in separate lemmas. Lemma~\ref{lem:normbounds} then follows by a union bound.

\begin{lemma}
\label{l:proofofpart1}
For all $\kappa > 0$, there is a $c \geq 1$ such that, for all large enough $C$, with probability at least $1-\delta/6$, for all $k \in [n]$,$$\frac{p}{c} \le \lv z_k \rv^2 \le c p.$$
\end{lemma}
\begin{proof}
For any \emph{clean} sample $z_i$, the random variables $(z_{ij}-\mu_j)^2$ are sub-exponential with norm $$\lv (z_{ij}-\mu_j)^2\rv_{\psi_1} \le  \lv z_{ij}-\mu_j \rv_{\psi_2}^2 \le 1.$$ After centering, the sub-exponential norm of the zero-mean random variable $(z_{ij}-\mu_j)^2-\mathbb{E}[(z_{ij}-\mu_j)^2]$ is at most a constant \citep[see][Exercise~2.7.10]{vershynin2018high}. Therefore, by Bernstein's inequality, there is an absolute constant 
$c_0$ such that
\begin{align*}
\Pr[\lvert \lv z_i -\mu \rv_2^2 -\mathbb{E}[\lv z_i -\mu \rv_2^2] \rvert \ge t] \le  2\exp\left(-c_0\min\left\{\frac{t^2}{p},t\right\}\right).
\end{align*}
By setting $t=\kappa p/2$
\begin{align*}
\Pr[\lvert \lv z_i -\mu \rv_2^2 -\mathbb{E}[\lv z_i -\mu \rv_2^2] \rvert \ge \kappa p/2] \le \frac{\delta}{6n},
\end{align*}
since $p\ge C\log(n/\delta)$ for a large constant $C$. 
Recall that by assumption, $\kappa p \le\mathbb{E}[\lv z_i -\mu \rv_2^2]\le3p$, where the upper bound follows from the
assumption that the components of $q$ have sub-Gaussian norm at
most $1$. 
Recalling that $\kappa \leq 1$,
\begin{align}
\label{e:near.mean}
\frac{\kappa p}{2}\le\lv z_i - \mu \rv^2 \le 4 p
\end{align}
with probability at least $1-\delta/(6n)$. 

By Young's inequality for products, $\lv z_i-\mu \rv^2 \le 2\lv z_i \rv^2 + 2\lv \mu \rv^2$. 
Also recall that by assumption $\lv \mu\rv^2 < p/C$. Combining this with the left hand side in the display above, for large enough
$C$, we have
\begin{align*}
\lv z_i \rv^2 \ge \frac{1}{2}\left(\frac{\kappa p}{2}-2\lv \mu\rv^2\right) 
  > \frac{\kappa p}{8}.
\end{align*}
Again by Young's inequality $\lv z_i \rv^2 = \lv z_i - \mu + \mu \rv^2 \le 2\lv z_i - \mu\rv^2 +2\lv \mu\rv^2$. Therefore, 
\begin{align*}
\lv z_i \rv^2 \le 2\lv z_i -\mu\rv^2 + 2\lv \mu \rv^2 \le 8 p + 2\lv \mu\rv^2 < 10 p,
\end{align*}
with the same probability. 

A similar argument also holds for all noisy samples by considering the random variables $(z_k-(-\mu))$. Hence, by taking a union bound over all samples completes the proof.
\end{proof}

\begin{lemma}\label{l:proofofpart2}
There is a $c \geq 1$ such that,
for all large
enough $C$,
with probability at least $1-\delta/6$, for all $i\neq j \in [n]$,
\begin{align*}
\lvert z_i \cdot z_j\rvert < c(\lv \mu \rv^2 + \sqrt{p \log(n/\delta)}).\label{event:2}\\
\end{align*}
\end{lemma}
\begin{proof}
First, let us condition on the division of $\{ 1,\ldots,n \}$ into
clean points $\cC$ and noisy points $\cN$.  After this, for each
$i \in \cC$, $\E[z_i] = \mu$ (where the expectation is conditioned on $z_i$ being clean), and for each $i \in \cN$,
$\E[z_i] = -\mu$.  For each $i$, let $\xi_i := z_i - \E[z_i]$.  
The same logic that proved \eqref{e:near.mean}, together with
a union bound, yields
\begin{equation}
\label{e:norms.small}
\Pr\left[\exists i \in [n],\; \lv \xi_j \rv > 2 \sqrt{p} \right] \leq \frac{\delta}{24}.
\end{equation}

For any pair $i,j \in [n]$ of indices of examples, we have
\begin{align}
\label{e:separate}
\Pr\left[\lvert \xi_i \cdot \xi_j \rvert \ge t\right] 
 \leq \Pr\big[\lvert \xi_i \cdot \xi_j \rvert \ge t \;\big|\; 
      \lv \xi_j \rv \leq 2 \sqrt{p} \big]
      + \Pr\left[ \lv \xi_j \rv > 2 \sqrt{p} \right].
\end{align}
If we regard $\xi_j$ as fixed, and only $\xi_i$ as random,
Theorem~\ref{thm:hoeffding} gives
\[
\Pr\big[\lvert \xi_i \cdot \xi_j \rvert \ge t \big]
 \leq 2\exp\left( -c_2 \frac{ t^2 }{ \lv \xi_j \rv^2 } \right).
\]
Thus,
\[
\Pr\big[\lvert \xi_i \cdot \xi_j \rvert \ge t \;\big|\; 
      \lv \xi_j \rv \leq 2 \sqrt{p} \big]
  \leq 2\exp\left( -c_2 \frac{ t^2 }{ 4 p } \right)
  = 2\exp\left( -c_3 \frac{ t^2 }{ p } \right)
\]
for $c_3 = c_2/4$.
Substituting into inequality~\eqref{e:separate}, we infer
\[
\Pr\left[\lvert \xi_i \cdot \xi_j \rvert \ge t\right]
 \leq 2\exp\left( -c_3 \frac{ t^2 }{ p } \right)
    + \Pr\left[ \lv \xi_j \rv > 2 \sqrt{p} \right].
\]
Taking a union bound over all pairs for the first term, and all
individuals for the second term, we get
\[
\Pr\left[\exists i \neq j \in [n],\; \lvert \xi_i \cdot \xi_j \rvert \ge t\right]
 \leq 2n^2 \exp\left( -c_3 \frac{ t^2 }{ p } \right)
    + \Pr\left[\exists j \in [n],\; \lv \xi_j \rv > 2 \sqrt{p} \right].
\]
Choosing $t = c \sqrt{p \log (n/\delta)}$ for 
a large enough value of $c$, 
we have
\[
\Pr\left[\exists i \neq j \in [n],\; \lvert \xi_i \cdot \xi_j \rvert \ge c \sqrt{p \log (n/\delta)}\right]
 \leq \frac{\delta}{24} + \Pr\left[ \exists j \in [n],\; \lv \xi_j \rv > 2 \sqrt{p} \right].
\]
Applying \eqref{e:norms.small}, we get
\begin{equation}
\label{e:all.xis}
\Pr\left[\exists i \neq j \in [n], \; \lvert \xi_i \cdot \xi_j \rvert \ge c  \sqrt{p \log (n/\delta)}\right]
 \leq \frac{\delta}{12}.
\end{equation}
For any $i$, clean or noisy, Hoeffding's inequality implies
\begin{align*}
\Pr\left[\lvert \mu\cdot z_i \rvert > 2\lv \mu\rv^2\right] 
   < 2\exp\left(-c_2 \frac{ 4 \lv \mu \rv^4 }{ \lv \mu \rv^2} \right)
   = 2\exp\left(- 4 c_2 c^2 \lv \mu \rv^2\right).
\end{align*}
Since $\lv \mu \rv^2 \geq  C \log(n/\delta)$, this implies
\begin{align*}
\Pr\left[\lvert \mu\cdot z_i \rvert > 2 \lv \mu\rv^2\right] < \frac{\delta}{12n}.
\end{align*}
Therefore, by taking a union bound over all $i \in \{ 1,\ldots,n \}$
\begin{align} \label{eq:normconc}
\Pr\left[\exists i, \; \lvert \mu\cdot z_i \rvert > 2\lv \mu\rv^2\right] < \delta/12.
\end{align}
Both the events in \eqref{e:all.xis} and \eqref{eq:normconc} will simultaneously hold with probability at most $\delta/6$. Assume that the event complementary to this bad event occurs, then for any distinct pair $z_i$ and $z_j$ 
\begin{align*}
\lvert z_i \cdot z_j \rvert
& = \big\lvert (z_i - \E[z_i])\cdot(z_j-\E[z_j]) + \E[z_i] \cdot \E[z_j] + \mu\cdot z_i + \mu\cdot z_j\big\rvert \\
& = \lvert \xi_i \cdot \xi_j + \E[z_i] \cdot \E[z_j] + \mu\cdot z_i + \mu\cdot z_j\rvert \\
&\le  \lvert \xi_i \cdot \xi_j \rvert + \lv \mu \rv^2 + \lvert \mu\cdot z_i \rvert + \lvert \mu\cdot z_j \rvert\\
& \le 5 \lv \mu\rv^2 + c \sqrt{p \log(n/\delta)},
\end{align*}
which completes our proof.
\end{proof}

\begin{lemma}
\label{l:part3}
For all large enough $C$, with probability at least $1 - \delta/6$,
\begin{align*}
\text{for all }k\in \mathcal{C}, \; \lvert \mu \cdot z_k -\lv \mu \rv^2\rvert < \lv \mu\rv^2/2.\\
\end{align*}
\end{lemma}
\begin{proof}
If $z_k$ is a clean point then, $\mathbb{E}[z_k\lvert k \in \mathcal{C}] = \mu$. Therefore the random variable $\lvert \mu \cdot z_k -\lv \mu\rv^2 \rvert = \lvert \mu \cdot (z_k -\mu) \rvert$ has sub-Gaussian norm at most $\lv \mu\rv$. By applying Hoeffding's inequality,
\begin{align*}
\Pr[\lvert \mu \cdot z_k -\lv \mu \rv^2 \rvert \ge \lv \mu \rv^2 /2] \le \frac{\delta}{6n},
\end{align*}
since $\lv \mu\rv^2 > C \log(n/\delta)$. Taking a union bound over all clean points establishes the claim.
\end{proof}

\begin{lemma}
\label{l:part4}
For all large enough $C$, with probability at least $1 - \delta/6$,
\begin{align*}
\text{for all }k\in \mathcal{N}, \; \lvert \mu \cdot z_k - (-\lv \mu \rv^2)\rvert < \lv \mu\rv^2/2.\\
\end{align*}
\end{lemma}
\begin{proof}
The proof is the same as the proof of Lemma~\ref{l:part3}, except that for any noisy sample, $z_k$, the conditional mean $\mathbb{E}[z_k \lvert k~\in~\mathcal{N}] = -\mu$. 
\end{proof}

\begin{lemma}
\label{l:part5}
For all $c' > 0$, for all large
enough $C$, with probability $1 - \delta/6$ the number of noisy samples satisfies $\lvert \mathcal{N} \rvert \le  (\eta + c') n$.
\end{lemma}
\begin{proof}
Since $n \geq C \log(1/\delta)$, this follows from a Hoeffding bound.
\end{proof}

\begin{lemma}
\label{l:sep}
If \eqref{event:1} and \eqref{event:2} hold,
then, if $C$ is large enough,
$(x_1,y_1),\ldots,(x_n,y_n)$ are linearly separable.
\end{lemma}
\begin{proof}
Let $v := \sum_{k=1}^n z_k$.  For each $k$ and any $\delta > 0$,
\begin{align*}
y_k v \cdot x_k 
   & = \sum_i y_i y_k x_i \cdot x_k \\
   & \ge p/c_1 - \sum_{i \neq k} y_i y_k x_i \cdot x_k \\
   & \geq p/c_1 - c_1 n (\lv \mu \rv^2 + \sqrt{p \log(n/\delta)}) \\
   & > 0
\end{align*}
for $p \geq C \max\{\lv \mu\rv^2 n, n^2 \log(n/\delta) \}$, completing
the proof.
\end{proof}

\section{Decreasing Loss}
\label{a:n.bound}

\begin{lemma}
\label{l:n.bound}
For all small enough step sizes $\alpha$, for all iterations $t$,
$R(v^{(t)}) \leq n$.
\end{lemma}
\begin{proof}
Since $R(v^{(0)}) = \sum_{j \in [n]}\exp\left(0\cdot z_j\right)= n$, it suffices to prove that, for all $t$, $R(v^{(t+1)}) \leq R(v^{(t)})$.
Toward showing this, note that,
if $c_1$ is the constant $c$ from Lemma~\ref{lem:normbounds}, the operator norm of the Hessian at any solution $v$ may be bound as follows:
\begin{align*}
\lv \nabla^2 R(v) \rv
     & = \left\lv \sum_k z_k z_k^{\top} \exp(-v z_k)\right\rv  \\
     & \leq  \sum_k \left\lv z_k z_k^{\top}  \right\rv \exp(-v z_k)  \\
     & \leq c_1 p \sum_k \exp(-v z_k) \\
     & = c_1 p R(v).
\end{align*}
This implies that $R$ is $c_1 p n$-smooth over those
$v$ such that $R(v) \leq n$.  This implies that, for
$\alpha \leq (c_1 p n)^{-1}$, if $R(v^{(t)}) \leq n$ then $R(v^{(t+1)}) \leq R(v^{(t)}) \leq n$ \citep[this can be seen, for example, by mirroring the argument used in Lemma~B.2 in][]{ji2018risk}. The lemma then
follows using induction.
\end{proof}



\printbibliography
\end{document}